\documentclass{article}

\usepackage{PRIMEarxiv}

\usepackage[utf8]{inputenc}
\usepackage[T1]{fontenc}
\usepackage{hyperref}
\usepackage{url}
\usepackage{booktabs}
\usepackage{array}
\usepackage{amsfonts}
\usepackage{amsmath}
\usepackage{amsthm}
\usepackage{mathtools}
\usepackage{nicefrac}
\usepackage{microtype}
\usepackage{fancyhdr}
\usepackage{graphicx}
\graphicspath{{media/}}

\newtheorem{lemma}{Lemma}
\newtheorem{corollary}{Corollary}
\newtheorem{proposition}{Proposition}

\title{MeRoTune: RoPE-Safe Merging with a Tunable Dial}

\author{
  Salman Faroz \\
  \texttt{stsfaroz@gmail.com} \\
}

\begin{document}
\maketitle

\begin{abstract}
When you merge two fine-tuned models from the same base checkpoint by simply averaging their weights, you implicitly assume their attention subspaces are still aligned. Recent work \cite{scalinglmc2026} attempts to fix misalignments by learning an invertible correction matrix, $M$, for each model's query and key projections. This correction cancels out—using $M$ on the query side and $M^{-\top}$ on the key side—right before the dot product. However, this cancellation is only exact if nothing sits between the projection and the dot product. In reality, almost all modern open-weight language models put a rotary position embedding (RoPE) exactly there. In this paper, we show that this cancellation is exact under RoPE if and only if $M$ commutes with RoPE's per-position rotation. We derive the specific class of matrices where this holds: a scaled rotation acting independently within each RoPE frequency pair. This forms a strict, low-dimensional subset of the unconstrained matrices that current methods normally train. Building on this, we turn this constrained matrix class into a new merging method. While keeping the base weights entirely frozen, two fine-tunes each learn their own RoPE-compliant correction matrices. We optimize these corrections against a chosen blend ratio so the final result can be adjusted post-hoc like a dial, rather than locked into a single fixed merge. Our default approach trains at one fixed blend ratio, similar to how LoRA sets its scaling hyperparameter in advance \cite{hu2021lora}. We also experiment with resampling the blend ratio randomly at every training step, and we report the results of both approaches. 
Additionally, we define a closed mathematical condition to determine whether this alignment process is even worth doing in the first place—specifically, checking for a genuine, bidirectional capability gap between the two fine-tunes—and we empirically verify this condition before training instead of just assuming it. We test our method on a verified pair of Qwen2.5-1.5B-Instruct fine-tunes (an Indonesian/code specialist and a Japanese specialist). Our approach never loses to the mergekit reference implementations of TIES and DARE-TIES at any blend ratio tested. It also clears our own statistical noise floor on most benchmarks (with the exception of one, whose margins remain inside that floor throughout), whereas DARE-TIES actually falls behind the unmerged base model on at least one axis. Finally, we explore whether the tied blend ratio can be split into two independent scaling factors, finding that doing so reliably degrades performance across all benchmarks as both factors rise. Code, training data, and merge/evaluation scripts are publicly available \cite{releaserepo}.
\end{abstract}

\section{Introduction}

Model merging combines several independently fine-tuned copies of one
base checkpoint into a single set of weights, usually to obtain a model
competent at every fine-tune's specialty without the cost of serving
several models or retraining from scratch \cite{mergingsurvey2026}.
The dominant family of methods -- weight averaging, task arithmetic
\cite{ilharco2022taskarithmetic}, TIES \cite{yadav2023ties}, DARE
\cite{yu2023dare}, SLERP \cite{shoemake1985slerp}, Model Stock
\cite{jang2024modelstock} -- all operate directly on raw weight tensors.
None of them ask whether two fine-tunes' internal coordinate systems
for a given layer actually agree before combining them.

A separate line of work asks exactly that question for attention:
before averaging two models' query/key projections, learn a per-head
invertible correction $M$ and fold it in as $W_q \!\to\! MW_q$,
$W_k \!\to\! M^{-\top}W_k$. For any invertible $M$ this is an exact
no-op in isolation -- $(Mq)^\top(M^{-\top}k) = q^\top M^\top M^{-\top} k
= q^\top k$ -- so nothing is lost by inserting it, and averaging two
models' $M$-corrected projections can undo a rotational mismatch
averaging their raw projections cannot. Li and Shen use exactly this
construction, with $M$ left completely unconstrained (a Cayley-transform
orthogonal factor times a Cholesky-factor symmetric part), to scale
linear mode connectivity to billion-parameter transformers
\cite{scalinglmc2026}.

The cancellation $q^\top M^\top M^{-\top}k = q^\top k$ holds because
nothing sits between the projection and the dot product. Rotary
position embedding (RoPE) \cite{su2021roformer} -- used by Qwen
\cite{qwen2025technical}, Llama 3 \cite{touvron2024llama3}, and
essentially every current open-weight decoder -- inserts a
position-dependent rotation there on purpose, to make the attention
logit depend on relative position. An $M$ chosen without this in mind
has no reason to leave that rotation's action alone, and one that
does not commute with it changes the merged model's relative-position
behavior in a way invisible to a smooth training loss that never
checks for it. Section~\ref{sec:transforms} makes this precise and
derives the exact matrices for which the cancellation survives RoPE:
not all invertible $M$, but a specific $2$-parameter family acting
within each RoPE frequency pair.

We use this family -- which we call the \emph{RoPE-commutant} class --
to build MeRoTune (Merging + RoPE + Tunable): a dual-sided alignment
where each fine-tune learns its own RoPE-legal correction against a
blend ratio, so the result works as a runtime dial rather than a
single fixed-ratio checkpoint. Our default fixes that blend ratio
before training, the same way LoRA's own scaling hyperparameter is
fixed in advance \cite{hu2021lora}, rather than treating it as
something that must be resampled during training; we also tried
resampling it randomly and compare both. We also give a
closed condition, checked before any training is done, for when this
kind of correction addresses a real problem at all rather than
searching for structure that is not there. Contributions:
\begin{enumerate}
  \item A short, self-contained proof of which query/key corrections
    survive RoPE exactly (Section~\ref{sec:transforms}), and the
    dual-sided, alpha-scannable merging method built on it
    (Section~\ref{sec:method}).
  \item A mathematical statement of when two fine-tunes have a genuine,
    bidirectional capability gap worth aligning for, checked empirically
    on five candidate models before any training is committed
    (Section~\ref{sec:setup}) -- rejecting three superficially
    plausible pairs in the process.
  \item An evaluation against the official mergekit implementations of
    TIES and DARE-TIES \cite{mergekit,mergekitgithub} -- not plain
    averaging, which is itself the trivial ($M{=}I$) member of the same
    commutant class and assumes zero misalignment rather than
    correcting one -- across the blend-ratio range, a comparison
    between fixing the blend ratio during training (our default) and
    resampling it randomly, and an ablation testing whether the tied
    blend ratio can be relaxed into two independent scaling factors
    (Section~\ref{sec:results}).
\end{enumerate}

\section{Background}

\subsection{Model merging}
Given a base checkpoint with parameters $\theta_0$ and fine-tunes
$\theta_1,\dots,\theta_k$, task arithmetic \cite{ilharco2022taskarithmetic}
defines the merge as $\theta_0 + \sum_i \lambda_i(\theta_i-\theta_0)$.
TIES \cite{yadav2023ties} refines this by trimming each task vector to
its top-magnitude coordinates and electing a single sign per
coordinate before summing, to reduce destructive interference between
task vectors. DARE \cite{yu2023dare} randomly zeroes a large fraction
of each task vector's coordinates and rescales the rest before
combination (commonly applied on top of TIES, as DARE-TIES). SLERP
\cite{shoemake1985slerp} interpolates two checkpoints along the
geodesic on the weight sphere rather than linearly. All of these treat
each weight tensor as an undifferentiated blob of numbers; none ask
whether two tensors being combined are expressed in the same internal
basis.

\subsection{Rotary position embedding}
RoPE \cite{su2021roformer} splits each attention head's $d$-dimensional
query and key vectors into $d/2$ pairs and rotates pair $j$ at token
position $m$ by angle $m\theta_j$, with frequencies
$\theta_j = \mathrm{base}^{-2j/d}$. We use the split-half convention
(pair $j$ is coordinates $(j, j{+}d/2)$), matching Hugging Face's
implementation and, in particular, Qwen2.5. Writing $R(m\boldsymbol\theta)$
for the resulting block-diagonal rotation, the query and key at
positions $m,n$ are $q_m = R(m\boldsymbol\theta)W_qx_m$,
$k_n = R(n\boldsymbol\theta)W_kx_n$, and the attention logit is
\begin{equation}
  q_m^\top k_n = x_m^\top W_q^\top R(m\boldsymbol\theta)^\top R(n\boldsymbol\theta)
  W_k x_n = x_m^\top W_q^\top R\big((n{-}m)\boldsymbol\theta\big) W_k x_n,
  \label{eq:rope-logit}
\end{equation}
using $R(m\boldsymbol\theta)^\top R(n\boldsymbol\theta) = R((n{-}m)\boldsymbol\theta)$,
the rotation-composition identity that makes the logit a function of
relative position alone.

\section{RoPE-Compatible Query-Key Transforms}
\label{sec:transforms}

\subsection{The cancellation, and what breaks it}
Consider correcting one model's query/key projections with an
invertible matrix $M$ (per attention head, applied identically to
every head unless stated otherwise): $W_q \!\to\! MW_q$,
$W_k \!\to\! M^{-\top}W_k$. Without RoPE, the corrected logit is
$q^\top M^\top M^{-\top} k = q^\top(M^\top)(M^\top)^{-1}k = q^\top k$
exactly, for \emph{any} invertible $M$ -- this is the entire appeal of
the construction: it is a free, exact no-op in isolation, so folding it
into one model's weights before averaging with another model's
$M'$-corrected weights can only help, never hurt, that model's own
behavior.

With RoPE inserted between the projection and the dot product, using
\eqref{eq:rope-logit} with $q'=Mq$, $k'=M^{-\top}k$, the corrected
logit at relative position $\Delta=n-m$ becomes
\begin{equation}
  q^\top M^\top R(\Delta\boldsymbol\theta) M^{-\top} k,
  \label{eq:corrected-logit}
\end{equation}
which equals the uncorrected logit $q^\top R(\Delta\boldsymbol\theta)k$
for every $\Delta$ if and only if
$M^\top R(\Delta\boldsymbol\theta) M^{-\top} = R(\Delta\boldsymbol\theta)$
for every $\Delta$, i.e.\ $R(\Delta\boldsymbol\theta)$ commutes with
$M^{-\top}$ for every $\Delta$. Taking transposes and using
$R(\Delta\boldsymbol\theta)^\top = R(-\Delta\boldsymbol\theta)$ shows
this holds for every $\Delta$ exactly when $M$ itself commutes with
$R(\Delta\boldsymbol\theta)$ for every $\Delta$. An $M$ chosen without
this constraint -- as in \cite{scalinglmc2026} -- cancels exactly only
in the degenerate case $M \propto I$; for a generic learned $M$,
\eqref{eq:corrected-logit} silently distorts the relative-position
dependence RoPE is there to provide, in a way a training loss that
never re-derives \eqref{eq:rope-logit} cannot see.

\subsection{The commutant of a single RoPE pair}

\begin{proposition}
\label{prop:pair-commutant}
Let
\[
  R(\theta) = \begin{pmatrix}\cos\theta & -\sin\theta \\ \sin\theta & \cos\theta\end{pmatrix}
\]
for some fixed $\theta \notin \{0,\pi\} \pmod{2\pi}$. A real $2{\times}2$
matrix $M$ satisfies $MR(\theta) = R(\theta)M$ if and only if
$M = aI + bJ$ for some $a,b \in \mathbb{R}$, where
\[
  J = \begin{pmatrix}0&1\\-1&0\end{pmatrix}.
\]
\end{proposition}
\begin{proof}
($\Leftarrow$) $J = R(\pi/2)$, and rotations commute with each other,
so $aI+bJ$ commutes with $R(\theta)$ for every $\theta$.
($\Rightarrow$) Write
\[
  M=\begin{pmatrix}p&q\\r&s\end{pmatrix}, \qquad c=\cos\theta,\quad u=\sin\theta.
\]
Expanding $MR(\theta)=R(\theta)M$
entrywise gives $(q{+}r)u=0$ from the $(1,1)$ and $(2,2)$ entries and
$(s{-}p)u=0$ from the $(1,2)$ and $(2,1)$ entries. Since
$\theta\notin\{0,\pi\}$, $u\neq0$, forcing $r=-q$ and $s=p$, i.e.\
$M = pI + qJ$.
\end{proof}

Because $J^2=-I$, the algebra $\{aI+bJ\}$ is isomorphic to $\mathbb{C}$
(with $J$ playing the role of $i$): commutant elements are exactly
scaled rotations, det $= a^2+b^2 \ge 0$, so every one of them (other
than $a=b=0$) is invertible and orientation-preserving.

\subsection{The commutant of a full RoPE head}
RoPE applies $d/2$ independent rotations $R(m\theta_j)$, one per
frequency pair, with $\theta_j=\mathrm{base}^{-2j/d}$ for
$j=0,\dots,d/2{-}1$ and a standard base (e.g.\ $10^4$). For any such
base and any $d$ used in practice, these frequencies are pairwise
distinct, and since each $\theta_j\in(0,1]$ (radians), no two of them
can sum to a multiple of $2\pi$ either -- both properties used below.
$M\in\mathbb{R}^{d\times d}$ commuting with $R(\Delta\boldsymbol\theta)$
for every integer $\Delta$ is, blockwise, $d/2\times d/2$ conditions of
the form $M_{jk}R(\Delta\theta_k)=R(\Delta\theta_j)M_{jk}$ (one
$2{\times}2$ block per pair $(j,k)$); the diagonal blocks ($j{=}k$) are
exactly Proposition~\ref{prop:pair-commutant}, and the off-diagonal
ones ($j{\neq}k$) are ruled out entirely by the following.

\begin{lemma}
\label{lem:cross-block}
Let $\alpha,\beta\in\mathbb{R}$ with $\alpha\not\equiv\beta\pmod{2\pi}$
and $\alpha+\beta\not\equiv0\pmod{2\pi}$. The only real $2{\times}2$
matrix $N$ satisfying $NR(\Delta\beta)=R(\Delta\alpha)N$ for every
integer $\Delta$ is $N=0$.
\end{lemma}
\begin{proof}
Identify $\mathbb{R}^2$ with $\mathbb{C}$ via $(x,y)\mapsto x+iy$; under
this identification $R(\theta)$ acts as multiplication by $e^{i\theta}$,
and every real-linear map on $\mathbb{C}$ is uniquely $z\mapsto
pz+q\bar z$ for some $p,q\in\mathbb{C}$ -- a relabeling of $N$'s four
real entries as two complex ones. The hypothesis
$N(e^{i\Delta\beta}z)=e^{i\Delta\alpha}N(z)$ for every $z\in\mathbb{C}$
expands to $pe^{i\Delta\beta}z+qe^{-i\Delta\beta}\bar z =
pe^{i\Delta\alpha}z+qe^{i\Delta\alpha}\bar z$; since this $(z,\bar z)$
decomposition is unique, the $z$- and $\bar z$-coefficients must match
separately:
\[
  p\,e^{i\Delta\beta} = p\,e^{i\Delta\alpha}, \qquad
  q\,e^{-i\Delta\beta} = q\,e^{i\Delta\alpha}, \qquad \text{for every integer }\Delta.
\]
Already at $\Delta{=}1$, the first equation forces $p=0$ (as
$\alpha\not\equiv\beta\pmod{2\pi}$) and the second forces $q=0$ (as
$\alpha+\beta\not\equiv0\pmod{2\pi}$), so $N=0$.
\end{proof}

Lemma~\ref{lem:cross-block} applies to every off-diagonal block
$M_{jk}$ ($\alpha{=}\theta_j$, $\beta{=}\theta_k$, both hypotheses
satisfied by the distinctness and boundedness noted above), forcing
$M_{jk}=0$. Hence the RoPE-commutant of a full head is exactly the set
of block-diagonal matrices with $d/2$ independent $2{\times}2$
scaled-rotation blocks -- a $d$-dimensional subspace of the
$d^2$-dimensional space of invertible matrices Li and Shen's
construction searches over.

\begin{corollary}
\label{cor:solo-invariance}
If $M$ is in the RoPE-commutant class, applying the correction to one
model in isolation ($M$ on $W_q$, $M^{-\top}$ on $W_k$, nothing on the
other side of a merge) reproduces that model's original outputs
exactly, for \emph{any} choice of $M$'s parameters. The correction has
no effect until it is combined with a \emph{different} model's
correction.
\end{corollary}
This follows directly from \eqref{eq:corrected-logit} being an
identity for commuting $M$. It matters for Section~\ref{sec:method}:
a learned RoPE-commutant transform cannot, by construction, change a
single model's own behavior -- whatever it learns can only be about
how to reconcile two models' bases with each other.

\begin{corollary}
\label{cor:conditioning}
Every nonzero $2{\times}2$ block $M=aI+bJ$ of the RoPE-commutant class
has condition number exactly $1$: both singular values equal
$\sqrt{a^2+b^2}$, for every $(a,b)$ gradient descent could ever reach,
not just at initialization.
\end{corollary}
\begin{proof}
$J^\top=-J$ and $J^2=-I$, so
$M^\top M = (aI-bJ)(aI+bJ) = a^2I + ab(J{-}J) - b^2J^2 = (a^2+b^2)I$.
A matrix with $M^\top M$ a scalar multiple of $I$ has both singular
values equal to the square root of that scalar.
\end{proof}
This class can never become ill-conditioned during training regardless
of where $(a_j,b_j)$ drift, in contrast to a general invertible
correction (as in \cite{scalinglmc2026}), whose conditioning is
unconstrained and can in principle degrade arbitrarily as its
parameters move.

\section{MeRoTune}
\label{sec:method}

Figure~\ref{fig:overview} sketches the method end to end; the rest of
this section derives and justifies each piece of it.

\begin{figure}[h]
\centering
\includegraphics[width=\textwidth]{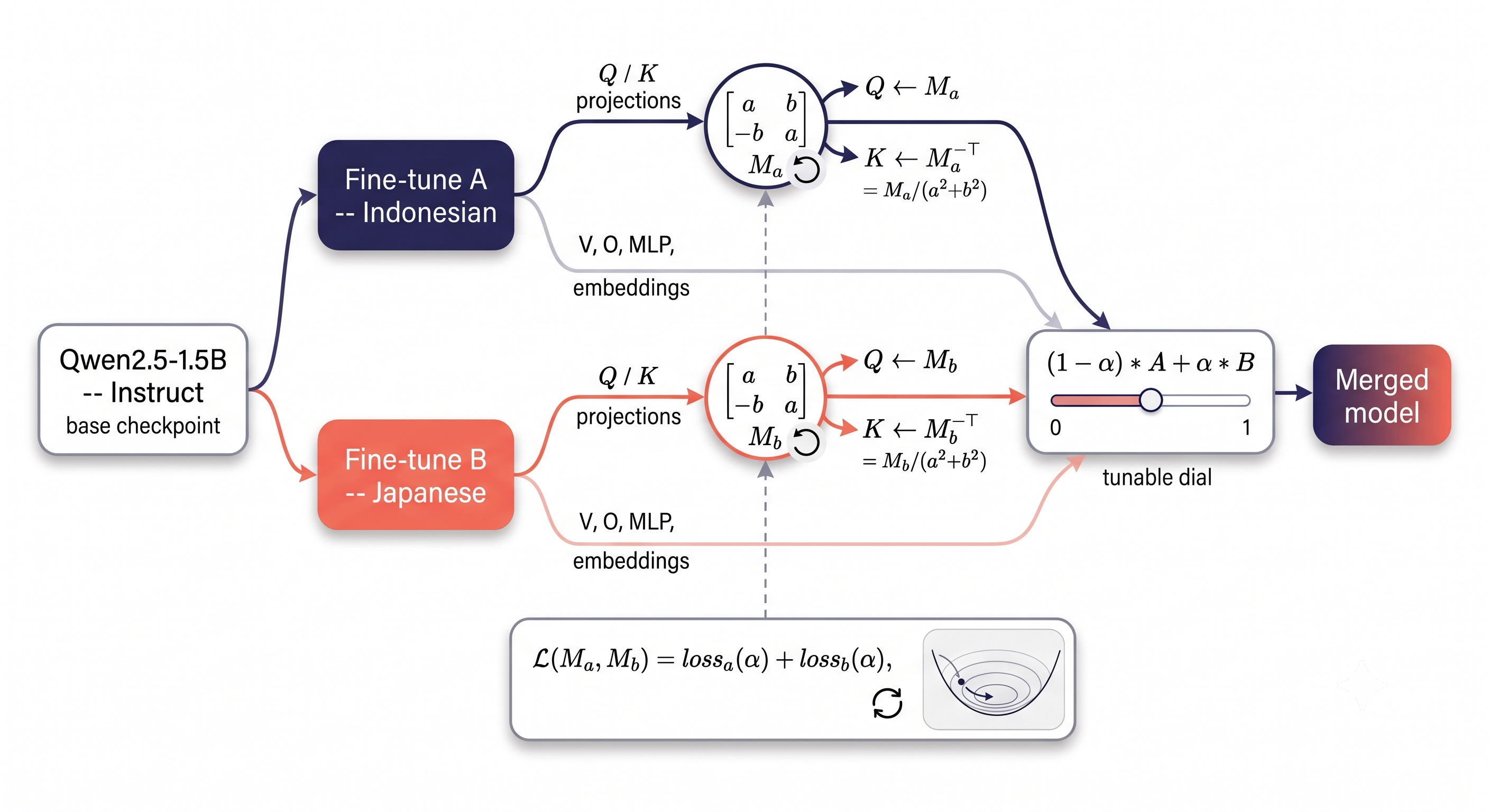}
\caption{MeRoTune end to end. Each fine-tune's Q/K projections pass
through its own learned RoPE-commutant transform ($M_a$ or $M_b$,
Section~\ref{sec:transforms}); every other tensor is untouched at this
stage. The two transformed models are then combined at a chosen blend
ratio $\alpha$ (Eqs.~\eqref{eq:mergeK}--\eqref{eq:mergeQ}), a dial that
can be swept after training rather than fixed at merge time.}
\label{fig:overview}
\end{figure}

Two structural conditions must hold before any of the following
applies. First, $A$ and $B$ must be full-parameter fine-tunes of the
\emph{same} base checkpoint -- identical architecture, tokenizer, and
vocabulary -- so that every weight tensor in
\eqref{eq:mergeK}--\eqref{eq:mergeQ} is shape-compatible between the
two models; an adapter-only checkpoint with no merged dense weights,
or a fine-tune that alters the tokenizer or vocabulary size, is not
mergeable this way regardless of how well it performs. Second, the
architecture must actually use RoPE, since that is what
Section~\ref{sec:transforms} constrains against. Both conditions were
applied as a hard filter on candidate fine-tunes before any benchmark
was run (Section~\ref{sec:setup}): a Russian candidate was excluded for
a replaced tokenizer (vocabulary size $145{,}152$ against the base's
$151{,}936$), a German PII-redaction candidate for being distributed
as a LoRA adapter with no merged checkpoint, and a diffusion-language-model
candidate for using a different model class entirely. Only checkpoints
passing this structural filter were evaluated against the statistical
condition, \eqref{eq:mutualgap} below.

\subsection{Parameterization}
We parameterize the commutant class per (layer, key/value head) with
one scalar pair $(a_j,b_j)$ for each of the $d/2$ RoPE frequency
pairs, $M_j = \begin{psmallmatrix}a_j&b_j\\-b_j&a_j\end{psmallmatrix}$,
initialized to $a_j{=}1,b_j{=}0$ (identity). This is
\texttt{ConstrainedM} in the released code \cite{releaserepo}; for
Qwen2.5-1.5B-Instruct ($28$ layers, $2$ key/value heads, head dim
$128$) it has $28\times2\times64\times2=7168$ trainable scalars in
total. Table~\ref{tab:paramcount} compares this against the two
unconstrained alternatives it is meant to replace: a raw, freely
trained per-head matrix (the natural LoRA-style patch), and Li and
Shen's own $QP$ parameterization \cite{scalinglmc2026}, which uses the
same $d^2$ real degrees of freedom per head as the raw matrix -- its
Cayley/Cholesky construction buys a particular kind of numerical
convenience, not a smaller search space. Only the RoPE-commutant
class gets its safety guarantee (Proposition~\ref{prop:pair-commutant})
and perfect conditioning (Corollary~\ref{cor:conditioning}) from the
parameterization itself, at roughly $128\times$ fewer parameters per
head.

\begin{table}[h]
\centering
\caption{Per-head query/key correction, three constructions.}
\label{tab:paramcount}
\begin{tabular}{lccc}
\toprule
construction & params/head & RoPE-safe & condition number \\
\midrule
ConstrainedM (ours)          & $d=128$   & yes, by construction & exactly $1$ \\
raw unconstrained matrix     & $d^2{=}16{,}384$ & no & unbounded \\
$QP$ \cite{scalinglmc2026}   & $d^2{=}16{,}384$ & no & not guaranteed \\
\bottomrule
\end{tabular}
\end{table}

\subsection{Dual-sided merge operator}
We adopt Li and Shen's dual-sided matching \cite{scalinglmc2026} --
both models learn their own correction rather than one being fixed as
an anchor -- but neither their $M$ nor their matching is RoPE-legal:
our version restricts $M_a,M_b$ to the commutant class of
Section~\ref{sec:transforms}, and trains them against the blend ratio
$\alpha$ of Section~\ref{sec:training} so the pair works as a dial at
merge time rather than a fixed-point correction. Concretely, $A$ and $B$ each get their own transform,
$M_a,M_b$; for a chosen blend ratio $\alpha\in[0,1]$, the merged
query/key weights at a given (layer, head) are
\begin{align}
  W_k^{\mathrm{merge}}(\alpha) &= (1-\alpha)\,M_a^{-\top}W_k^A
    + \alpha\,M_b^{-\top}W_k^B, \label{eq:mergeK}\\
  W_q^{\mathrm{merge}}(\alpha) &= (1-\alpha)\,M_aW_q^A
    + \alpha\,M_bW_q^B, \label{eq:mergeQ}
\end{align}
with $k/q$-proj biases treated identically (they are added before
RoPE, so they ride along with the weight they sit next to), and every
other tensor (embeddings, $V$/$O$ projections, MLP, norms) plain
linearly interpolated at the same $\alpha$. By
Corollary~\ref{cor:solo-invariance}, $\alpha{=}0$ reproduces fine-tune
$A$ exactly and $\alpha{=}1$ reproduces fine-tune $B$ exactly,
regardless of the learned values of $M_a,M_b$; the interior of $[0,1]$
is where the two models' K/Q subspaces, each first rotated into a
jointly-learned shared frame, are actually combined. Plain weight
averaging is the special case $M_a=M_b=I$: a legal member of the same
commutant class, but one that \emph{assumes} the two models' frames
already agree rather than searching for and correcting a real
misalignment. When that assumption happens to hold it can look
perfectly adequate, which is precisely the failure mode: nothing in
plain averaging detects whether it holds. If the two models' Q/K
frames genuinely disagree, Section~\ref{sec:transforms} already shows
what happens next -- averaging corrects nothing, and the mismatch
distorts RoPE's relative-position structure in exactly the way
\eqref{eq:corrected-logit} makes precise, invisibly to any training
loss that never re-checks it.

This kind of correction is not always worth applying. In plain terms:
it is worth it only when each fine-tune is genuinely good at its own
specialty \emph{and} genuinely not good at the other's -- not when a
fine-tune is mediocre even at its own job, and not when one fine-tune
already beats the other at both. The condition that captures this
exactly is stated next.
\label{sec:mutual-gap}
Let $D_A,D_B$ be held-out sets for $A$'s and $B$'s specialties and
$L(\theta;D)$ a task error rate. Define, relative to the shared base
$\theta_0$,
\begin{equation}
  \Delta_A = L(\theta_0;D_A)-L(\theta_A;D_A), \quad
  \Delta_B = L(\theta_0;D_B)-L(\theta_B;D_B),
\end{equation}
\begin{equation}
  \Gamma_A = L(\theta_0;D_B)-L(\theta_A;D_B), \quad
  \Gamma_B = L(\theta_0;D_A)-L(\theta_B;D_A).
\end{equation}
$\Delta_A,\Delta_B$ are each fine-tune's own-domain gain over base;
$\Gamma_A,\Gamma_B$ are their cross-domain shift on the \emph{other}
fine-tune's task. A pair is a genuine target for alignment-then-merge
exactly when
\begin{equation}
  \Delta_A>0,\ \ \Delta_B>0,\ \ \Gamma_A\le0,\ \ \Gamma_B\le0:
  \label{eq:mutualgap}
\end{equation}
each model is a real improvement on its own specialty and no better
(typically worse) on the other's. If \eqref{eq:mutualgap} fails --
$\Delta_i\le0$ for some $i$ (a fine-tune that is not actually better
than base at anything, an ``alignment tax'') or $\Gamma_i>0$ for some
$i$ (one fine-tune already dominates the other on both tasks) -- there
is no real rotational disagreement to correct: the two models are not
each carrying genuine, conflicting specialization, and any learned
$M_a,M_b$ can only add parameters and training noise to a problem that
does not exist. We check \eqref{eq:mutualgap} empirically for every
candidate pair \emph{before} committing to training
(Section~\ref{sec:setup}).

\subsection{Training objective}
\label{sec:training}
With $\theta(\alpha;M_a,M_b)$ denoting the parameters assembled by
\eqref{eq:mergeK}--\eqref{eq:mergeQ} at blend ratio $\alpha$, and
$\ell$ the model's own next-token loss on answer tokens, training
minimizes
\begin{equation}
  \mathcal{L}(M_a,M_b) = \mathbb{E}_{\alpha\sim U[\alpha_{\min},\alpha_{\max}]}
  \Big[\, \ell\big(\theta(\alpha;M_a,M_b);\,x_a\big)
       + \ell\big(\theta(\alpha;M_a,M_b);\,x_b\big) \Big],
  \label{eq:loss}
\end{equation}
$x_a\sim D_A^{\mathrm{train}}$, $x_b\sim D_B^{\mathrm{train}}$. The two
loss terms are \emph{unweighted} by $\alpha$: weighting by, say,
$(1-\alpha)$ and $\alpha$ would make the gradient on $M_b$ vanish
whenever a step happens to sample $\alpha$ near $0$, so $M_b$ would
never receive a usable training signal at exactly the operating points
where it needs to already be correct.

$\alpha_{\min}=\alpha_{\max}$ recovers a single fixed blend ratio,
sampled with probability $1$ every step -- this is our default,
$\alpha_{\min}{=}\alpha_{\max}{=}0.5$, the same design choice LoRA
makes with its own scaling hyperparameter: fixed once before training
rather than varied during it \cite{hu2021lora}. We separately tried
the wider setting $\alpha_{\min}{=}0.05,\ \alpha_{\max}{=}0.95$,
resampled fresh every step, on the reasoning that a single trained
pair should then be valid as a dial across the whole range rather than
only near where it was trained. Both are trained identically otherwise
(same steps, learning rate, batch); Section~\ref{sec:results} reports
both. A separate exploratory check trained several further
single-fixed-ratio variants spanning the available range: training
near either extreme of that range generalized poorly, working only
close to that same extreme, while every more moderate choice
generalized across the whole range without collapsing. The likely
mechanism is that an extreme ratio starves one model's correction of
meaningful gradient for the entire run, while a moderate ratio keeps
both engaged throughout.

Gradients with respect to $(a_j,b_j)$ flow through
\eqref{eq:mergeK}--\eqref{eq:mergeQ} by ordinary reverse-mode autodiff
over the frozen base network (\texttt{torch.func.functional\_call});
the only non-trivial local derivatives are $\partial M_j/\partial a_j
= I$, $\partial M_j/\partial b_j=J$, and the standard matrix-inverse
identity $\partial(M_j^{-1})/\partial\theta = -M_j^{-1}(\partial
M_j/\partial\theta)M_j^{-1}$ for the key-side term -- no custom
backward pass is required, and none of it touches the frozen
backbone's own parameters.

\subsection{Decoupling the blend ratio}
\label{sec:independent}
The convex form \eqref{eq:mergeK}--\eqref{eq:mergeQ} ties the two
weights so they always sum to one: turning up how much of $B$ is used
necessarily turns down how much of $A$ is used by the same amount. A
task-arithmetic-style generalization relaxes exactly this: give each
fine-tune its own independent dial, $s_a$ and $s_b$, for how much of
its own change from the shared base to add in, with no requirement
that the two sum to anything in particular:
\begin{equation}
  \theta(s_a,s_b) = \theta_0 + s_a\big(\phi(\theta_A;M_a)-\theta_0\big)
                            + s_b\big(\phi(\theta_B;M_b)-\theta_0\big),
  \label{eq:independent}
\end{equation}
where $\phi(\theta;M)$ denotes $\theta$ with its K/Q weights replaced
by their $M$-rotated form. Setting $s_a{=}1{-}\alpha,\ s_b{=}\alpha$
recovers \eqref{eq:mergeK}--\eqref{eq:mergeQ} exactly (the two
$\theta_0$ terms cancel), so \eqref{eq:independent} is a strict
generalization, not a different method. Nothing in the construction
guarantees that pushing both $s_a,s_b$ high at once approaches both
fine-tunes' own single-model scores simultaneously -- summing two
uncontrolled deltas with no TIES-style conflict resolution could just
as easily interfere destructively. Section~\ref{sec:sasb-ablation}
tests this directly, reusing already-trained $M_a,M_b$ with no
retraining.

\section{Experimental Setup}
\label{sec:setup}

\subsection{Base model and fine-tunes}
All experiments use Qwen2.5-1.5B-Instruct \cite{qwen2025technical}
(28 layers, 12 query heads, 2 key/value heads, head dim 128,
grouped-query attention with group size 6, bias on $q$/$k$ projections)
as the shared base. Fine-tune $A$ is \texttt{arissuga/aurum-brain-ai}
\cite{aurummodel} (Indonesian and code). Fine-tune $B$ is
\texttt{SakanaAI/TinySwallow-1.5B-Instruct} \cite{tinyswallowmodel}
(Japanese, TAID-distilled from a Qwen2.5-32B-Instruct teacher). Out of
five Qwen2.5-1.5B-Instruct candidates checked against
\eqref{eq:mutualgap} on a native-language subject-knowledge benchmark
per language (INCLUDE \cite{include2024} for Indonesian/Japanese/Russian,
C-Eval \cite{huang2023ceval} for Chinese, $n{=}100$) -- which also
included a Russian specialist and a Chinese text-correction specialist
-- $A,B$ was the only pair satisfying it; full numbers are in the
released repository \cite{releaserepo}.

\subsection{Training data}
Fine-tune $A$'s retention set is a small Indonesian/code instruction
set used for this fine-tune throughout our broader experiments.
Fine-tune $B$'s retention set is 220 filtered examples from
\texttt{databricks-dolly-15k-ja} \cite{dollyja2023}, a Japanese
translation of Databricks Dolly 15k -- deliberately \emph{not} $B$'s
own TAID-distillation data (to avoid training the retention transform
on the exact distribution it was distilled from) and deliberately
general-purpose rather than code-specific (to avoid confounding with
$A$'s own specialty).

\subsection{Baselines}
Plain weight averaging ($\alpha$-interpolation, no rotation, the
$M{=}I$ member of our own commutant class). TIES and DARE-TIES via the
official \texttt{mergekit} toolkit \cite{mergekit,mergekitgithub}
(density $0.5$, weight $0.5/0.5$ each side, base-relative,
\texttt{bfloat16}; run on CPU, as \texttt{--cuda} exceeded an 8\,GB
card's memory during TIES's magnitude-sparsification step). We
compare against TIES and DARE-TIES, not plain averaging, as the
serious baseline: plain averaging is a passive reference point within
our own method's parameter space, not an independent competing tool.

\subsection{Benchmarks and training configuration}
MMLU \cite{hendrycks2021mmlu} (57 subtasks, $n{=}100$ each,
$5{,}700$ total) and ARC-Challenge \cite{clark2018arc} ($n{=}100$) via
lm-evaluation-harness \cite{gao2024lmeval}; INCLUDE-Indonesian and
INCLUDE-Japanese \cite{include2024} are evaluated on their full
available splits ($n{=}450$ and $n{=}299$ respectively), not a
$100$-example subsample -- larger than the other two tasks, which
matters for the noise-floor discussion in
Section~\ref{sec:limitations}. Training: 500 steps,
batch size 2 examples per task per step, AdamW, learning rate $3{\times}
10^{-3}$, gradient-norm clip $1.0$; $\alpha_{\min}{=}\alpha_{\max}{=}0.5$
(fixed) by default, and separately $\alpha\sim U[0.05,0.95]$ resampled
every step for the randomized comparison. All numbers in this paper were produced with the public
reference implementation \cite{releaserepo}, so every table below is
directly reproducible from it.

\section{Results}
\label{sec:results}

\subsection{Main comparison}
Table~\ref{tab:main} is the agreed comparison: base, both fine-tunes
alone, TIES, DARE-TIES, and our method -- trained with a single fixed
blend ratio, $\alpha_{\min}{=}\alpha_{\max}{=}0.5$ -- evaluated at five
merge-time blend ratios.

\begin{table}[h]
\centering
\caption{Indonesian/Japanese pair, four benchmarks. ``Ours'' trains
$M_a,M_b$ at fixed $\alpha{=}0.5$ and merges at the listed $\alpha$.
\emph{A:B} is the blend fraction assigned to fine-tune $A$ (Indonesian)
vs.\ fine-tune $B$ (Japanese): $(1{-}\alpha){:}\alpha$ for our method,
the configured mergekit weight for TIES/DARE-TIES, and $1{:}0$/$0{:}1$
for a fine-tune alone. Bold marks the best of our method's five rows
in that column (ties bolded together) -- see
Section~\ref{sec:results} for how these compare against TIES.}
\label{tab:main}
\begin{tabular}{lccccc}
\toprule
condition & A:B & MMLU & ARC-C & Indonesian & Japanese \\
\midrule
base                        & --      & 0.611 & 0.43 & 0.536 & 0.539 \\
fine-tune $A$ alone         & 1.0:0.0 & 0.613 & 0.42 & 0.536 & 0.542 \\
fine-tune $B$ alone         & 0.0:1.0 & 0.567 & 0.41 & 0.456 & 0.602 \\
mergekit TIES               & 0.5:0.5 & 0.553 & 0.42 & 0.442 & 0.508 \\
mergekit DARE-TIES          & 0.5:0.5 & 0.465 & 0.30 & 0.353 & 0.318 \\
\emph{ours}, $\alpha{=}0.1$ & 0.9:0.1 & \textbf{0.608} & \textbf{0.46} & \textbf{0.527} & 0.555 \\
\emph{ours}, $\alpha{=}0.3$ & 0.7:0.3 & 0.599 & \textbf{0.46} & 0.511 & 0.579 \\
\emph{ours}, $\alpha{=}0.5$ & 0.5:0.5 & 0.585 & 0.43 & 0.496 & 0.562 \\
\emph{ours}, $\alpha{=}0.7$ & 0.3:0.7 & 0.581 & \textbf{0.46} & 0.496 & 0.599 \\
\emph{ours}, $\alpha{=}0.9$ & 0.1:0.9 & 0.576 & 0.42 & 0.478 & \textbf{0.619} \\
\bottomrule
\end{tabular}
\end{table}

TIES lands below the untouched base model on three of four
benchmarks; DARE's additional drop-and-rescale step on top of TIES
makes every column noticeably worse still, not better. Against TIES --
the real tool to beat -- ours does not lose to TIES on a single cell
in this table. Bold cells mark the best of our own five blend ratios
per column, not a claim about the margin over TIES specifically:
ARC-C is the one column where even that best point's margin over TIES
($0.04$) stays inside this paper's own noise floor
($\mathrm{SE}\approx0.05$ at $n{=}100$, Section~\ref{sec:limitations}),
so that column is better read as a wash than a win. MMLU, Indonesian,
and Japanese margins clear their respective (larger-sample) noise
floors at most blend ratios and are what this comparison actually
rests on. We also trained the randomized-$\alpha$ variant of
Section~\ref{sec:training} on this same pair; it beat DARE-TIES at
every blend ratio tested, but not TIES outright -- it loses to TIES on
Japanese at $\alpha{=}0.5$ and on ARC-C at $\alpha{=}0.9$, by narrow
margins, against clearly wider winning margins everywhere else and
throughout the fixed-$\alpha$ run above (full numbers in the released
repository \cite{releaserepo}).

\subsection{What the learned transforms look like}
Table~\ref{tab:main}'s gains only mean what Section~\ref{sec:method}
claims if the trained $M_a,M_b$ actually moved away from the identity
they were initialized at. Every learned block is, by
Corollary~\ref{cor:conditioning}, exactly a scaled rotation -- that is
never in question, it holds by construction regardless of what
training does. What is an empirical question is how large the
rotation is, and whether it comes with a scale close to $1$ (a nearly
pure rotation) or a scale substantially different from $1$ (rotation
entangled with a real magnitude change). Table~\ref{tab:rotation-summary}
and Figure~\ref{fig:rotation-stats} give the answer, across all
$3{,}584$ (layer, kv-head, RoPE-pair) blocks per model.

\begin{table}[h]
\centering
\caption{Summary statistics for the learned blocks (fixed-$\alpha{=}0.5$ training).}
\label{tab:rotation-summary}
\begin{tabular}{lcccc}
\toprule
& median $|\theta|$ & 90th pctl. $|\theta|$ & max $|\theta|$ & \% $>45^\circ$ \\
\midrule
$M_a$ (Indonesian) & $9.7^\circ$ & $24.4^\circ$ & $86.2^\circ$ & $1.0\%$ \\
$M_b$ (Japanese)   & $9.8^\circ$ & $25.0^\circ$ & $75.2^\circ$  & $0.9\%$ \\
\bottomrule
\end{tabular}
\end{table}

The typical block barely moves: median rotation is about $10^\circ$
for both models, 90\% stay under $25^\circ$, and for these small-angle
blocks the scale is close to $1$ (mean $0.98$) -- close to a
pure rotation, as the theory motivates. But the large rotations (up to
$86^\circ$) belong to a small tail, under $1\%$ of blocks, and that
tail is {\bf not} a clean rotation: restricted to blocks with
$|\theta|{>}45^\circ$, mean scale drops to $0.55$ ($M_a$) and $0.60$
($M_b$) -- those specific blocks shrink vectors substantially on top
of rotating them. Figure~\ref{fig:rotation-stats} shows this is not
incidental: rotation angle and scale deviation are entangled across
the whole population, not independent. Nothing in \eqref{eq:loss}
penalizes moving far from identity, so we cannot rule out that this
small tail reflects unregularized optimizer capacity -- reshaping a
few heads' behavior more broadly than a pure geometric realignment --
rather than a confirmed, isolated misalignment correction at exactly
those blocks. The randomized-$\alpha$ run shows the same qualitative
pattern (median $\approx\!10^\circ$, a sub-2\% tail up to
$80$--$106^\circ$ with the same scale-shrinkage entanglement), so this
is not an artifact of which training regime was used.

\begin{figure}[h]
\centering
\includegraphics[width=\textwidth]{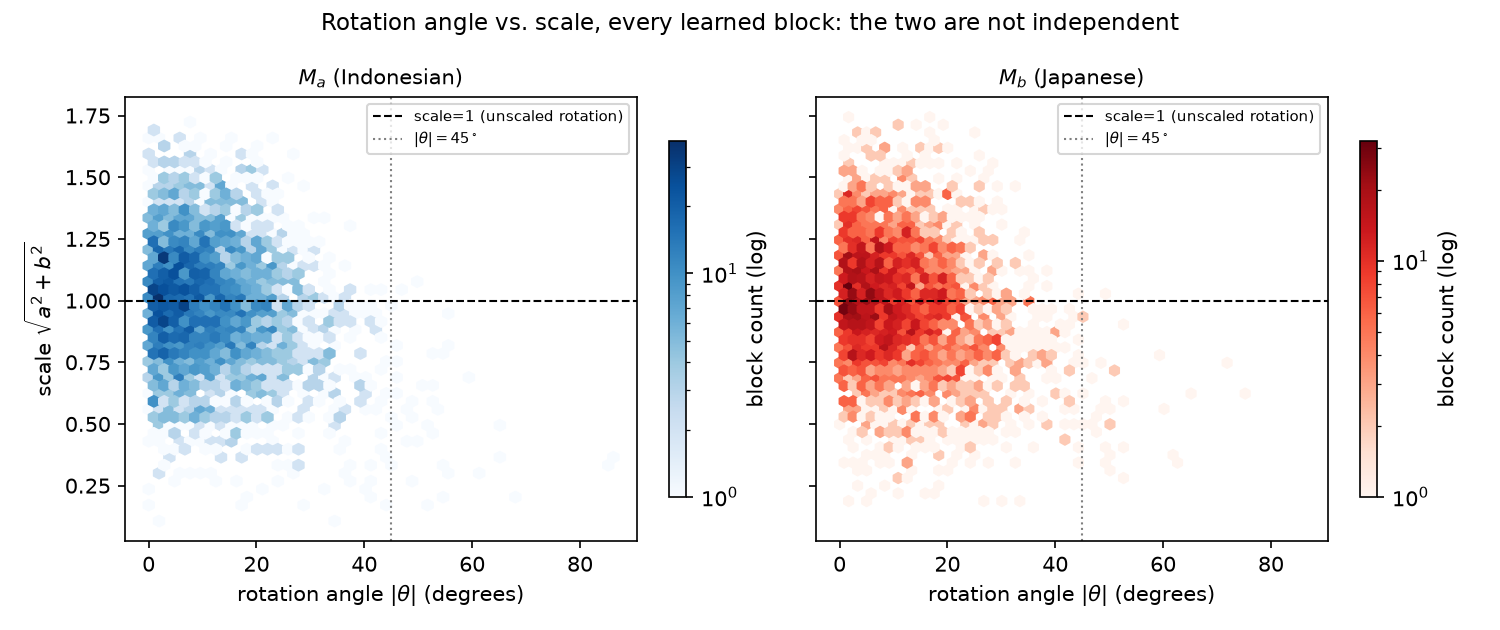}
\caption{Rotation angle vs.\ scale for every learned block. Most mass
sits at small angle with scale near $1$; the sparse tail beyond
$45^\circ$ sits almost entirely below scale $1$, showing the large
rotations are not scale-preserving.}
\label{fig:rotation-stats}
\end{figure}

\subsection{Alpha as a dial}
Table~\ref{tab:alphasweep} extends the fixed-$\alpha{=}0.5$ training's
merge-time sweep to all nine tenths-wide points.

\begin{table}[h]
\centering
\caption{Extended alpha sweep, fixed-$\alpha{=}0.5$ training, same four
benchmarks. \emph{A:B} is $(1{-}\alpha){:}\alpha$, the blend fraction
assigned to fine-tune $A$ (Indonesian) vs.\ fine-tune $B$ (Japanese).
Bold marks the best value in each column across all nine points.}
\label{tab:alphasweep}
\begin{tabular}{lccccc}
\toprule
$\alpha$ & A:B & MMLU & ARC-C & Indonesian & Japanese \\
\midrule
0.1 & 0.9:0.1 & \textbf{0.608} & \textbf{0.46} & 0.527 & 0.555 \\
0.2 & 0.8:0.2 & 0.604 & 0.44 & \textbf{0.529} & 0.572 \\
0.3 & 0.7:0.3 & 0.599 & \textbf{0.46} & 0.511 & 0.579 \\
0.4 & 0.6:0.4 & 0.593 & 0.44 & 0.502 & 0.579 \\
0.5 & 0.5:0.5 & 0.585 & 0.43 & 0.496 & 0.562 \\
0.6 & 0.4:0.6 & 0.581 & 0.43 & 0.496 & 0.569 \\
0.7 & 0.3:0.7 & 0.581 & \textbf{0.46} & 0.496 & 0.599 \\
0.8 & 0.2:0.8 & 0.578 & 0.44 & 0.502 & \textbf{0.632} \\
0.9 & 0.1:0.9 & 0.576 & 0.42 & 0.478 & 0.619 \\
\bottomrule
\end{tabular}
\end{table}

Alpha behaves as a scaling factor, not a precise dial: the broad trend
on each metric follows whichever fine-tune is stronger there. MMLU
(where $A$ is stronger) is exactly monotonically non-increasing across
all nine points; Indonesian (also $A$-favored) trends the same way but
with small bumps (e.g.\ $\alpha{=}0.2$ and $0.8$ tick back up); Japanese
($B$-favored) trends up overall but dips around $\alpha{=}0.5$ and
peaks at $\alpha{=}0.8$, not $0.9$; ARC-C, where $A$ and $B$ are nearly
tied, shows no consistent trend at all -- consistent with there being
no dominant model to trend toward. Anyone sweeping this table will
still see the expected direction clearly on the metrics with a real
gap; it just should not be read as a guarantee of strict monotonicity
at every tenth. We treat $\alpha$ as a hyperparameter to
sweep and commit to for a deployment's actual need, not a quantity
promised to move smoothly between two known-good endpoints. Because
merging at a new $\alpha$ from an already-trained $M_a,M_b$ costs a
forward pass over the merge equations and nothing else -- no
retraining -- this sweep is cheap to run on a small held-out sample
from the target deployment's own data before committing to a value,
rather than assumed from a benchmark run on a different task mix. The
randomized-$\alpha$ training of Section~\ref{sec:training}, evaluated
the same way, showed larger peak-to-trough swings on the metrics with
a real gap between $A$ and $B$ (Indonesian $1.5\times$, Japanese
$1.9\times$), essentially the same swing on ARC-C where neither model
dominates ($1.0\times$), and a smaller difference on MMLU ($1.3\times$)
-- not a uniform $1.5$--$2\times$ effect across every metric. Full
numbers are in the released repository \cite{releaserepo}, not
reproduced here.

Training itself is stable throughout: over 500 steps at fixed
$\alpha{=}0.5$, both retention losses fall and stay bounded
(Table~\ref{tab:losscurve}, mean over the first and last 20 steps),
consistent with the low-dimensional, always-invertible parameterization
of Section~\ref{sec:transforms}.

\begin{table}[h]
\centering
\caption{Training loss, start vs.\ end (mean of first/last 20 of 500 steps).}
\label{tab:losscurve}
\begin{tabular}{lcc}
\toprule
& start & end \\
\midrule
loss$_a$ (Indonesian retention) & 0.470 & 0.293 \\
loss$_b$ (Japanese retention)   & 1.053 & 0.952 \\
\bottomrule
\end{tabular}
\end{table}

\subsection{Decoupling the blend ratio}
\label{sec:sasb-ablation}
We evaluate \eqref{eq:independent} at $(s_a,s_b)\in\{(0.7,0.7),
(0.9,0.9),(1.0,1.0)\}$, reusing the already-trained $M_a,M_b$ from the
randomized-$\alpha$ run (Section~\ref{sec:training}) with no
retraining, to test whether pushing both coefficients high at once
approaches both fine-tunes' single-model scores simultaneously rather
than trading one off against the other. Table~\ref{tab:sasb} gives the
result.

\begin{table}[h]
\centering
\caption{Independent (non-tied) $s_a,s_b$.}
\label{tab:sasb}
\begin{tabular}{lcccc}
\toprule
$(s_a,s_b)$ & MMLU & ARC-C & Indonesian & Japanese \\
\midrule
$(0.7,0.7)$ & 0.477 & 0.40 & 0.367 & 0.421 \\
$(0.9,0.9)$ & 0.324 & 0.36 & 0.276 & 0.308 \\
$(1.0,1.0)$ & 0.284 & 0.29 & 0.264 & 0.298 \\
\bottomrule
\end{tabular}
\end{table}

The outcome is unambiguous and monotonic in the wrong direction: every
benchmark degrades as both coefficients rise together, well below any
tied-$\alpha$ point in Table~\ref{tab:main} and approaching
random-guessing territory on a four-way multiple-choice task by
$(1.0,1.0)$. Pushing $s_a,s_b$ both high does not yield ``the best of
both'' fine-tunes; it compounds two large, uncontrolled deltas with no
mechanism to resolve conflict between them, exactly the failure mode
TIES's sign-election and magnitude-trimming exist to prevent. The tied
convex constraint $s_a{+}s_b{=}1$ is not an arbitrary restriction of
\eqref{eq:independent} -- removing it does measurable harm.

\section{Discussion}

Two patterns recur across every fine-tune pair we have examined under
this framework, both visible already in the candidate check of
Section~\ref{sec:setup}. First, an ``alignment tax'': some fine-tunes score below base on every
single axis, including their own claimed specialty, a general
regression from aggressive full-parameter SFT rather than a genuine
capability trade-off -- $\Delta<0$ in the language of
Section~\ref{sec:mutual-gap}, disqualifying the pair outright regardless
of how it compares to any other candidate. Second, fine-tune $A$ shows
essentially zero gain over base on the harder, subject-knowledge
benchmark ($0.536$ vs.\ $0.536$) while showing a real gain on
Belebele-style reading comprehension in our earlier, unpublished
pre-check sweep -- light instruction-tuning changes surface fluency
more readily than it changes deep subject knowledge, which is exactly
why \eqref{eq:mutualgap} must be checked on the harder benchmark rather
than assumed from a fine-tune's stated specialty or from a shallower
benchmark alone.

This method deliberately restricts the learned correction to
query/key projections -- the one place RoPE imposes a rotation that
plain averaging can silently disrupt without a training loss ever
seeing it. Value, output, MLP, and embedding weights carry no
comparable structure, which is also why plain averaging is a
categorically safer default there than it is for K/Q: there is no
rotation for it to misalign in the first place. Whether those tensors
would still benefit from their own learned, unconstrained correction
is open -- nothing rotates them, so no RoPE-commutant restriction
would apply, but adding one would also give up the free safety
guarantee Proposition~\ref{prop:pair-commutant} gets K/Q for a
$d$-dimensional, rather than $d^2$-dimensional, search. We leave this
to future work rather than assuming an answer either way.

\section{Limitations}
\label{sec:limitations}

For this pair, DARE-TIES's extra drop-and-rescale step does not pay
off: it underperforms plain TIES on every benchmark tested
(Table~\ref{tab:main}), which is not the result we expected going in.
All results are for one
architecture (Qwen2.5-1.5B) and one verified mutual-gap pair; the
candidate check is necessary precisely because we found, and rejected,
superficially attractive pairs that failed it (Section~\ref{sec:setup}),
so the specific numbers in
Section~\ref{sec:results} should not be assumed to transfer to an
unrelated pair without re-running that check. The independent
$s_a,s_b$ results (Section~\ref{sec:sasb-ablation}) reuse $M_a,M_b$
trained only under the randomized-$\alpha$ objective, and only from a
single training run; we have not checked whether the fixed-$\alpha$
transforms behave the same way when decoupled, nor across multiple
seeds of either. The fixed-vs-randomized comparison itself
(Section~\ref{sec:training}) is also a single run of each: fixed-$\alpha{=}0.5$
training matched or beat the randomized variant at essentially every
blend ratio tested on this pair, but with one seed each we cannot rule
out that this reflects this particular pair or this particular pair of
training runs rather than a general property of fixing the blend
ratio.

Sample size, and so the binomial noise floor
$\mathrm{SE}\approx\sqrt{0.5\times0.5/n}$ on a single-task accuracy,
differs by task: ARC-C at $n{=}100$ gives $\mathrm{SE}\approx0.05$;
Indonesian at $n{=}450$ and Japanese at $n{=}299$ are both larger
samples, giving $\mathrm{SE}\approx0.024$ and $\mathrm{SE}\approx0.029$
respectively; MMLU aggregates 57 subtasks at $n{=}100$ each
($5{,}700$ examples total, $\mathrm{SE}\approx0.007$) and is the one
comparison in this paper solid enough to treat as conclusive on its
own. Table~\ref{tab:main} and Table~\ref{tab:alphasweep} bold margins
against this per-task floor rather than one number applied uniformly;
ARC-C is the one column where several per-task margins in
Section~\ref{sec:results} fall short of it and should be read as
suggestive rather than statistically established.

\section{Conclusion}

Rotary position embedding is not a minor detail you can simply work around after the fact when aligning two models' attention subspaces—it determines, exactly, which corrections are safe to apply at all. We derived this class of corrections in closed form and built a dual-sided merging method around it that trains as a post-hoc dial, rather than producing a single fixed-ratio checkpoint. Alongside this, we introduced a mathematical condition that can be checked before training to verify whether two fine-tunes actually disagree enough to be worth aligning in the first place. When tested on a verified pair, our default setup—which trains at a single fixed blend ratio, mirroring the design choice LoRA makes with its own scaling hyperparameter—never loses to the official reference implementations of TIES and DARE-TIES at any blend ratio tested, and it clears our own statistical noise floor on most benchmarks. A randomized-blend-ratio variant was also tested and beat DARE-TIES throughout, though not TIES outright. Finally, a further ablation confirms that the method's tied blend ratio is not an arbitrary restriction: decoupling it into two independently scaled factors degrades every benchmark, rather than combining the best of both fine-tunes.

\end{document}